%% file: root.tex
\documentclass[letterpaper, 10 pt, conference]{ieeeconf}  

\IEEEoverridecommandlockouts                              

\usepackage{graphics} 
\usepackage{epsfig} 
\usepackage{mathptmx} 
\usepackage{times} 
\usepackage{amsmath} 
\usepackage{amssymb}  
\usepackage{pifont}
\usepackage{booktabs}
\usepackage{subcaption}
\usepackage{adjustbox}
\usepackage{array}
\usepackage{multirow, booktabs}
\usepackage{makecell}
\usepackage[mathscr]{euscript}
\usepackage{float}
\usepackage{gensymb}
\usepackage{comment}
\usepackage[table,x11names]{xcolor} 
\definecolor{PaleBlue}{rgb}{0.87,0.92,1}
\definecolor{PaleGreen}{rgb}{0.87, 1.0, 0.92}
\usepackage{mathtools}
\usepackage[bookmarks=true, hidelinks]{hyperref}    
\hypersetup{
     colorlinks = true,
     urlcolor = black
}
\usepackage{todonotes}

\let\labelindent\relax
\usepackage{enumitem}

\newtheorem{proposition}{Proposition}[section]

\usepackage{algpseudocode}
\usepackage{cite}
\usepackage{booktabs}
\usepackage{listings}
\usepackage[linesnumbered,algoruled,boxed,vlined, noend]{algorithm2e}

\title{\LARGE \bf
LGMCTS: Language-Guided Monte-Carlo Tree Search \\ for Executable Semantic Object Rearrangement 
}

\author{Haonan Chang, Kai Gao, Kowndinya Boyalakuntla, Alex Lee, Baichuan Huang,  Harish Udhaya Kumar, \\ Jingjin Yu, Abdeslam Boularias
\thanks{The authors are with the Department of Computer Science,
        Rutgers University, 08854 New Brunswick, USA. This work is supported by NSF awards 1846043 and 2132972.}}
\begin{document}

\maketitle
\thispagestyle{empty}
\pagestyle{empty}

\begin{abstract}



We present LGMCTS, a framework that uniquely combines language guidance with geometrically informed sampling distributions to effectively rearrange objects according to geometric patterns dictated by natural language descriptions. LGMCTS uses Monte Carlo Tree Search (MCTS) to create feasible action plans that ensure executable semantic object rearrangement. 
We present a comprehensive comparison with leading approaches that use language to generate goal rearrangements independently of actionable planning, including Structformer, StructDiffusion, and Code as policies.
We also present a new benchmark, the Executable Language Guided Rearrangement (ELGR) Bench, containing tasks involving intricate geometry. With the ELGR bench, we show 
limitations of task and motion planning (TAMP) solutions that are purely based on Large Language Models (LLM) such as Code as Policies and Progprompt on such tasks. Our findings advocate for using LLMs to generate intermediary representations rather than direct action planning in geometrically complex rearrangement scenarios, aligning with perspectives from recent literature.
Our code and supplementary materials are accessible at \href{https://lgmcts.github.io/}{\textcolor{blue}{https://lgmcts.github.io/}}.


\end{abstract}

\input{introduction_new}

\input{related}

\input{problem}

\input{proposed}
\input{experiments}

\input{conclusion}

\bibliographystyle{IEEEtran}
\bibliography{Star} 
\end{document}

%% file: introduction_new.tex
\section{Introduction}
Everyday tasks, such as ``Set up the kitchen", involve organizing objects based on verbal instructions, a process that is intuitive for humans but that presents a significant challenge for robots. The semantic rearrangement problem seeks to empower robots with the ability to reorganize a scene according to linguistic descriptions. This challenge necessitates that robots comprehend the task through natural language, and address the corresponding Task And Motion Planning (TAMP) problem effectively.

Traditionally, solving this problem requires formalizing semantic rearrangement into a symbolic representation, clearly defining the goal configuration or constraints, and using formal planners such as STRIPS~\cite{strips} and PDDL~\cite{pddl2}, or search-based planners like MCTS~\cite{mcts} to devise a feasible plan. Although effective, this approach demands expert-level knowledge to abstract a problem into a formal representation, limiting accessibility for average users.

To overcome this challenge, numerous recent studies have sought to tackle the problem directly from linguistic inputs and RGB-D observations~\cite{clip_port,struct_diffusion,liu2022structformer}. One approach uses multi-modality transformers to establish a correlation between verbal descriptions and object positions using data generated from the simulation. Following work such as StructDiffusion~\cite{struct_diffusion} further improved this method by using a diffusion model to build the multi-modality solution. However, a common drawback of these methods is that they rely on an offline training stage, which makes them applicable only to trained object categories and spatial patterns.

\begin{figure}
    \centering
    \includegraphics[width=\linewidth]{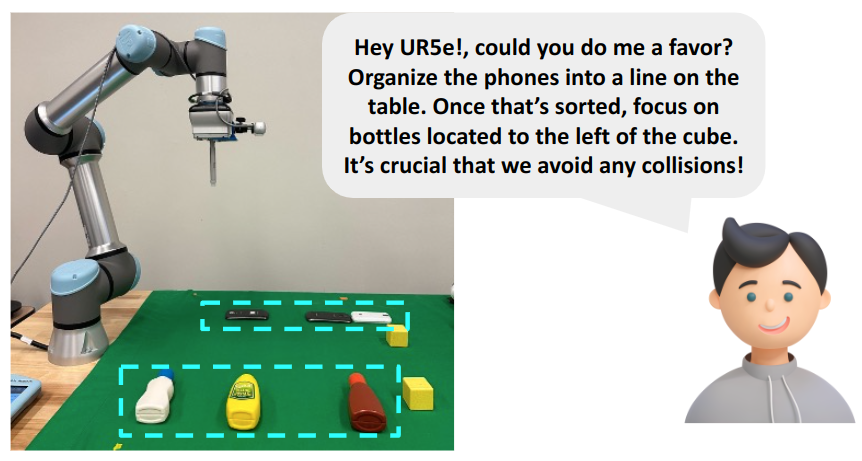}
    \caption{\small{Robotic Setup: a UR5e robot equipped with a RealSense D455 camera. The task is to re-arrange the objects, which are unknown to the robot, according to a natural language instruction.}}
    \label{fig:robot}
\end{figure}

With the advent of Large Language Models (LLMs), models such as GPT~\cite{brown2020language} and Llama~\cite{llama2} have demonstrated impressive potential in understanding complex scenarios and exhibiting zero-shot planning capabilities. This has led researchers to explore the utilization of LLMs in solving language-based TAMP problems~\cite{huang2022language,ahn2022can,code-as-policies}.  However, despite specific considerations for the feasibility of plans proposed by LLMs, it has been reported that these plans significantly lag in executability and completeness when compared to those crafted by a properly implemented traditional solver designed for the task~\cite{valmeekam2022large}. This observation has naturally led researchers to seek methods that merge the user-friendliness of LLMs with the robustness of traditional TAMP algorithms such as PDDL, STRIPS, or MCTS. LLM-GROP~\cite{ding2023task} follows this approach in rearrangement, employing LLMs to parse user tasks from language into pairwise spatial relationship specifications and then calling a sampling-based task and motion planner~\cite{zhang2022visually} to generate the plan. A limitation of LLM-GROP is that it can only handle pair-wise relationships, and thus cannot perform complex rearrangement tasks.
AutoTAMP~\cite{chen2023autotamp} uses LLMs to translate natural language into formal representations and then invokes a planner to tackle the problem. AutoTAMP can solve a wide range of TAMP tasks, but it does not apply to general semantic rearrangement where the action space is not discrete and potentially large.


We present Language-Guided Monte-Carlo Tree Search (LGMCTS), a new technique for executable semantic object rearrangement. Like its predecessors, AutoTAMP and LLM-GROP, LGMCTS leverages LLMs for generating intermediate representations and employs a planner for formulating feasible plans. A key novelty of LGMCTS is the integration of parametric geometric priors for spatial relationship representations. LGMCTS facilitates more nuanced handling of complex geometric relationships among multiple objects, addressing scenarios that require organization beyond simple pairwise interactions such as configurations in lines or rectangles. Additionally, LGMCTS takes a holistic approach by simultaneously considering task planning (goal specification) and motion planning (execution order and intermediate steps). During planning, an obstacle relocation strategy is used to handle obstacles that may block the execution.
This coordination ensures that plans are not only semantically coherent but also practically executable, offering a balanced consideration of goal achievement and operational efficiency.

To assess the efficacy of LGMCTS, we introduce the Executable Language-Guided Rearrangement (ELGR) benchmark, featuring over 1,600 varied language queries and robot execution checks. Our evaluations indicate that LGMCTS performs effectively on the ELGR benchmark, especially in comparison with Code as policies and Progprompt in terms of feasibility and semantic consistency of the generated goals. LGMCTS also outperforms Structformer and StructDiffusion in goal generation  on the Structformer dataset.

%% file: related.tex
\section{RELATED WORKS}

\subsection{Learning-based Semantic Rearrangement}
The semantic rearrangement problem consists of devising a rearrangement plan that is both semantically congruent with a given language description and physically feasible~\cite{tellex2011understanding,howard2014natural,chang2023contextawareentitygroundingopenvocabulary}. In recent years, this has gained increased traction, particularly as a pivotal application in language-driven robotics. CLIPort~\cite{clip_port} took the initial step in this direction by merging CLIP features with a Transporter network. Yet, its design is limited to basic pick-and-place tasks. 
Structformer~\cite{liu2022structformer} advanced the field using a transformer model, by simulating rearrangements with hand-crafted rules and connecting language tokens to object poses. Leveraging Structformer's dataset, StructDiffusion~\cite{struct_diffusion} introduced a pose diffusion model to predict poses from language.  Nonetheless, a common shortcoming amongst all these methodologies is their limitation to a single structure or pattern (e.g. circle, line) that they have been trained on, making composite patterns (e.g. rectangle + tower) a persistent challenge. Moreover, the rearrangement goals generated by these methods can be inexecutable. 

\subsection{LLM-driven Task And Motion Planning}
Recent advancements in LLMs~\cite{brown2020language, llama2} have showcased impressive performance across a broad spectrum of tasks. There has been a growing interest in using LLMs for TAMP~\cite{huang2023inner, ahn2022can, code-as-policies, singh2023progprompt, song2023llm, driess2023palm, lin2023text2motion, wu2023tidybot, huang2022language, ding2023task, huang2023voxposer, wu2023integrating, silver2022pddl, liu2023llm+, xie2023translating, rana2023sayplan, guan2024leveraging, llm_mcts, birr2024autogpt+}, owing to their few-shot and zero-shot reasoning ability \cite{brown2020language, COT, wei2022emergent}. Grounding language into a sequence of plannable tasks/actions without retraining LLMs was initially explored in~\cite{huang2022language}. 
Following this, SayCan~\cite{ahn2022can} was proposed to facilitate the conversion of LLM-generated plans into robot-executable steps, though it struggled with addressing task execution failures. 
Inner Monologue~\cite{huang2023inner} improved upon SayCan by incorporating real-time feedback to adjust plan post-execution, yet Inner Monologue is prone to generating suboptimal and infeasible plans. 
Instead of using LLMs to plan with predefined skills, other approaches such as Code as policies~\cite{code-as-policies} and Progprompt~\cite{singh2023progprompt} leveraged LLMs for policy code generation, showcasing their potential in behavioral common sense and sequential policy logic. However, they do not show promising results on complex object rearrangement tasks requiring more nuanced spatial context. This is due to the limitations of LLMs' planning ability over long horizon tasks~\cite{valmeekam2022large, liu2023llm+}. 

Owing to the drawbacks of the aforementioned works, to offer a more reliable and interpretable planning process, recent works~\cite{liu2023llm+, valmeekam2022large, chen2023autotamp, xie2023translating, guan2024leveraging} emphasize translating natural language commands into intermediate representations that are interpretable by traditional TAMP algorithms. 
Text2Motion~\cite{lin2023text2motion} uses LLMs to greedily plan a skill sequence combined with a geometric feasibility planner to ensure that the geometric dependencies are addressed. However, Text2Motion's hybrid LLM planner is less efficient in large spaces than planners such as MCTS~\cite{mcts}. 
LLM-GROP~\cite{ding2023task} translates language instructions to symbolic spatial relationships with LLMs and employs a task and motion planner named GROP~\cite{zhang2022visually} to perform the rearrangement. Although GROP is optimized for efficiency and feasibility, LLM-GROP is limited by its focus on simple object rearrangements due to its treatment of multi-object semantic relationships as pairwise reasoning. AutoTAMP~\cite{chen2023autotamp} employs LLMs for generating and validating STL representation from natural language and utilizes a formal STL planner~\cite{sun2022multi} for generating optimal trajectories. Although effective across a spectrum of tasks, its applicability is limited in non-discrete action spaces, as is the case in semantic rearrangement tasks.

To address these limitations, we introduce LGMCTS, a new approach that incorporates parametric geometric priors and that is guided by MCTS's efficient exploration to provide a feasible TAMP solution for complex rearrangement tasks.

%% file: problem.tex
\section{Preliminaries}

\subsection{Problem Formulation}
\label{sec:problem}
Semantic rearrangement is the task of rearranging a scene according to a series of natural language descriptions. One key insight here is that a goal described by language is usually a distribution rather than a single position. For example, ``Put the mug at the right side of bowl" refers to a uniform distribution among a region that is right to the bowl. If we know the position distribution for each object, we just need to sequentially sample the poses for each object. The semantic rearrangement problem is then converted to a sequential sampling problem.
\color{blue}


\color{black}

With this insight, we define the task of semantic rearrangement as follows.
The robot is given as input a scene with objects from a set 
$O_S = \{o_1, o_2,\ldots, o_N\}$
and a command \(L\), where \(L\) is a pure natural language command that implies a desired distribution list 
$\mathcal{F}=\{f_i: p(o_i) \sim f_{i} | o_i \in O_R\}$, where $p(o_i)$ refers to the position of object $o_i$.
Here, \(O_R \subseteq O_S\) denotes the objects requiring an action based on $L$, and \(f_i\) indicates the desired pose distribution for each object. The objective is to identify an optimal action sequence, $A = (a_t)_{t=1}^{H}$, where each action $a_t$ corresponds to moving an object $o_i$ to a sampled position $p(o_i)$,
with the objective to achieve a goal arrangement aligning $L$, i.e., $\prod_{o_i \in O_R} f_{i}(p_i) > 0$ and minimizing the number of action steps \(H\). Noticeably, \(A\) includes not only movements of objects \(o \in O_R\), but also those of distracting objects, denoted as \(O_D\), with \(O_D \subseteq O_S\).

\subsection{Monte Carlo Tree Search (MCTS)} \label{sec:mcts_frame}
We provide here a brief reminder of the MCTS~\cite{mcts} technique. A typical MCTS algorithm iteratively builds a search tree by performing the following four operations.
\begin{enumerate}[leftmargin=5mm]
    \item {\bf Selection.} On a fully expanded node (all the children nodes have been visited), MCTS selects to explore the branch  with the highest Upper Confidence Bound (UCB), 
    \begin{equation}
        \arg\max_a \left ( \dfrac{w(f(s,a))}{n(f(s,a))}+C\sqrt{\dfrac{\log(n(s))}{n(f(s,a))}} \right),
    \end{equation}
    where $f(s,a)$ is the child node of state $s$ after action $a$, $w(\cdot)$ and $n(\cdot)$ are respectively cumulative rewards and the number of visits to a state.
    \item {\bf Expansion.} On a node that is not fully expanded, MCTS selects an action that has not been attempted yet.
    \item {\bf Simulation.} Given a node and a selected action, MCTS simulates a sequence of actions and receives a reward.
    \item {\bf Back-Propagation.} MCTS passes the terminal reward to ancestor nodes to update their cumulative expected rewards, which indicate the quality of the branch.
\end{enumerate}
At each iteration, MCTS starts from the root node. When all the child nodes of the current node are visited, MCTS selects a child node with the UCB formula.
When some child nodes of the current node are unvisited, MCTS expands by randomly selecting a new action and performing a simulation to reach a new child node.
The new node returns a reward, which is back-propagated to all the ancestor nodes.

%% file: proposed.tex
\section{Method}
\label{sec:method}
In a nutshell, LGMCTS starts by calling an LLM to parse the language description $L$ to a list of spatial distributions $\mathcal{F}=\{f_i: p(o_i) \sim f_{i} | o_i \in O_R\}$. Then it uses an MCTS-based procedure to find a physically feasible action sequence $A$ to rearrange the scene according to distributions $\mathcal{F}$.

\subsection{Language Parsing \& Object Selection} \label{sec:parse}
During the language parsing stage, we parse $L$ into $\mathcal{F}=\{f_i: p(o_i) \sim f_{i} | o_i \in O_R\}$. Similar to previous methods~\cite{code-as-policies}, we conduct an automated prompt engineering to guide the LLM to perform the parsing. Fig.~\ref{fig:llm_parsing} showcases how prompt engineering is implemented. Essentially, the language model translates user requirements into structured goal configurations and constraints that guide task execution.

\begin{figure}[h]
    \centering
    \includegraphics[width=0.80\linewidth]{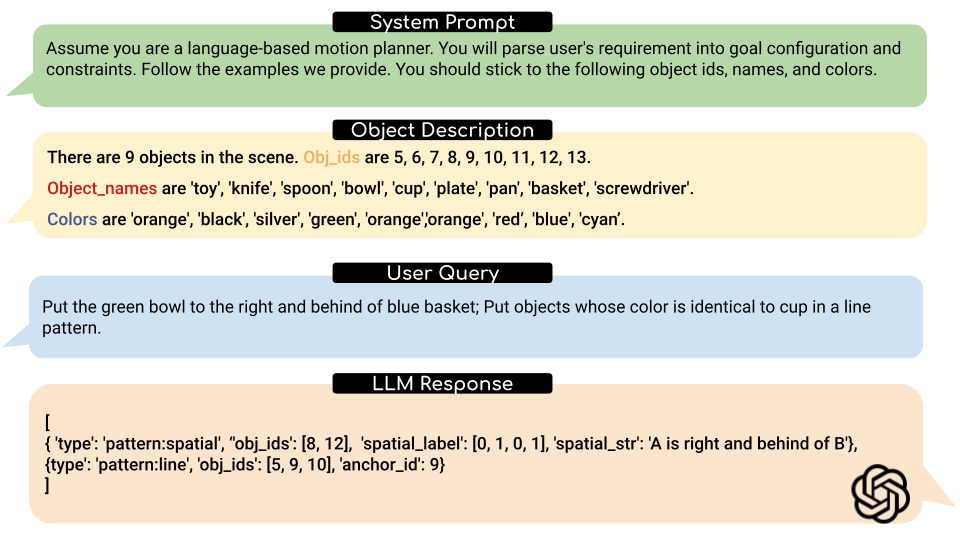}
    \caption{\small{An example of language parsing. We are using GPT-4~\cite{brown2020language} in this work.}}
    \label{fig:llm_parsing}
\end{figure}

Consider the example depicted in Fig.~\ref{fig:llm_parsing}. First comes the {\bf system prompt} providing guidelines for the LLM to follow when interpreting user queries. Then, we need to provide an {\bf object description} such as semantic labels, colors, and IDs for the objects in the scene. In practice, we use the Recognize Anything Model (RAM)~\cite{huang2023tag2text, zhang2023recognize} for producing the semantic labels, and a color detector to determine the colors of the objects in the scene. A unique ID is assigned to each object (using for the whole planning). Finally, we provide a {\bf user query} describing the rearrangement goal. A structured answer is returned from the LLM.
LLM's answers suggest how many patterns there are and which objects should be selected to form the pattern.

\subsection{Parametric Geometric Prior} 
\label{sec:prior}
\begin{figure}
    \centering
    \includegraphics[width=0.85\linewidth]{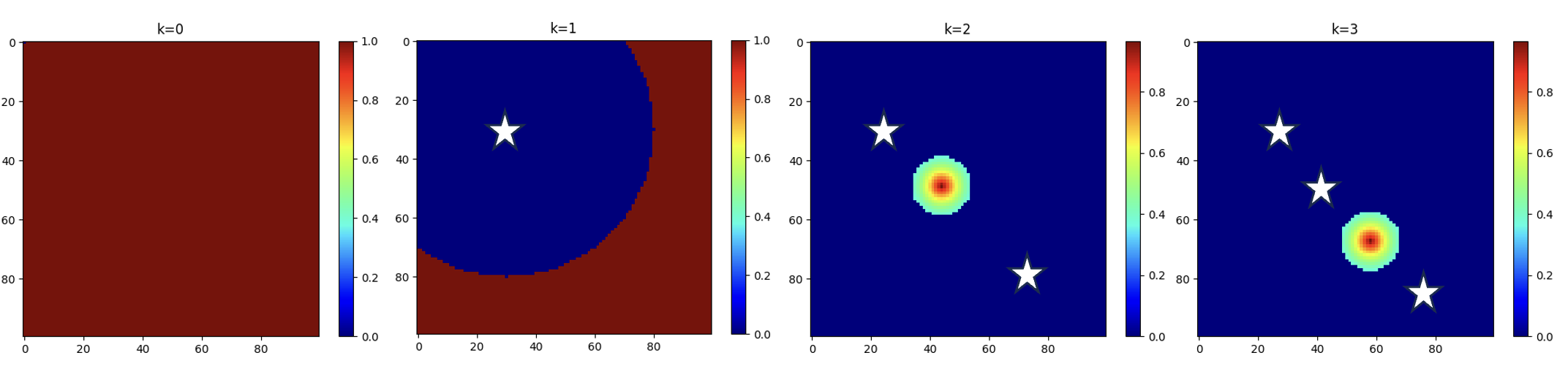}
    \caption{\small{Visualization of $(x,y)$ prior for `line' pattern. From left to right: $K=0$, $K=1$, $K=2$, $K=3$, where $K=|O_{R}^{sampled}|$, the number of sampled object poses. White star marks are sampled poses. When $K=0$, the pose can be sampled anywhere. When $K=1$, it needed to sampled outside a circle region. After that, all poses will be sampled along the line defined by the first two poses. }}
    \label{fig:prior}
    \vspace{-5mm}
\end{figure}

As mentioned in Section~\ref{sec:problem}, if we know the goal position distribution for each object, the semantic rearrangement problem can be converted to a sequential sampling problem. There are multiple details we need to pay attention to in this process. (1) The position distribution is not fixed but varies with the progress of sampling. For example, if we say objects A, B, and C need to be put into a line, and we sample in the order of A, B, and C, then the distributions of A and B are actually unbounded, and C must be placed on the line defined by the positions of A and B. (2) The position distribution should be collision-free. (3) One can build a database for many different spatial distributions, but we need to have a flexible mechanism to associate LLM's inferred distribution types with items in the database.

We show in the following how to compute the pose distribution function \( f_{i}(p_i|P_S, L) \) for each object  \( o_i \) in the set \( O_R \) during the sequential pose sampling process. Here, \( p_i \) denotes the pose \((x, y, \theta)\) of an object \( o_i \), \( P_S \) represents the list of current poses for all objects in the scene, and \( L \) is the natural language command. 

We compute \( f_{i}(p_i|P_S, L) \) as an element-wise product of two components, a pattern prior function \( f_{\text{prior}}(p_i|P_{R}, L) \) and a boolean function \( f_{\text{free}}(p_i|P_S) \) of the workspace,
\begin{equation}
    f_{i}(p_i|P_S, L) = f_{\text{prior}}(p_i|P_{R}, L) \times f_{\text{free}}(p_i| P_S)
\end{equation}
In this equation, \( P_{R} \) \color{black} is the list of current poses of all objects in  \( O_{R} \)\color{black}. 
\color{black}
We note that \( P_{R} \) is a subset to \( P_{S} \).
\color{black}
The function \( f_{\text{free}}(p_i|P_S) \) is set to $1$ \color{black} if $p_i$ is collision-free from the remaining poses in $P_S$, and to $0$ otherwise. \color{black} $f_{\text{free}}$ is determined by running a 2D collision simulation using point cloud observations for each object \( o_i \) \(\in\) \( O_S \).

Our primary focus is to determine \( f_{\text{prior}}(p_i|P_{R}, L) \). To this end, we employ an approach akin to the one described in \cite{ahn2022can}. We maintain a database comprising a collection of predefined prior functions. Each of these functions is linked with one or more Sentence-BERT embeddings, acting as keys. The function corresponding to the best matching key is selected, as follows,
\begin{equation*}
     f_{\text{prior}}(p_i|P_{R}, L) = f_{\text{prior}}^K(p_i|P_{R}) \textrm{ with } K = \underset{K\in \text{database}}{\text{argmax}}(\Theta_k \cdot \Theta({L}))
\end{equation*}
Here, \( \Theta_k \) is the \( K^{th} \) key in the database, and \( \Theta({L}) \) is the Sentence-BERT embedding generated from the language instruction \( L \).
In summary, the most suitable prior function from the database is selected based \color{black} on \( \Theta({L}) \) \color{black}.

We now delve into the definition of \( f_{\text{prior}}^K(p_i|P_{R}) \) in the context of our work. We use a unique model for representing various distributions by employing parametric curves. A parametric curve can be expressed as \( (x, y) = \gamma(t, \kappa) \), where \( t \) ranges from 0 to 1 and \( \kappa \) is a set of curve-defining parameters. In our work, \( \kappa \) is modeled as a function of two 2D positions, denoted as \( \kappa(p_0, p_1) \). Therefore, for each pattern, we define two functions: \( \gamma \) and \( \kappa \).

Given that pattern prior $f_{prior}$ is used inside a sequential sampling process (check Section~\ref{sec:mcts} for more details), the prior distribution needs to be iteratively updated to capture the history of sampling. Consequently, we further categorize \( O_{R} \) 
based on whether the objects have been sampled in the current branch of the MCTS-Planner. Subsets \( O_{R}^{sampled} \) and \( O_{R}^{unsampled} \) denote the sampled and non-sampled objects, respectively. Thus, $O_{R} = O_{R}^{sampled} \cup O_{R}^{unsampled}$.


The probability  \color{black} $f_{\text{prior}}^K(p_i|P_{R}))$ \color{black} of sampling a pose $p_i = (x_i, y_i, \theta_i)$ for next object \color{black} $o_i\in O_{R}^{unsampled}$ \color{black} is given as follows.
\begin{itemize}
    \item If \color{black} $|O_{R}^{sampled}|=0$ \color{black} then \( (x_i, y_i, \theta_i) \sim U \), suggesting that the first object can be placed arbitrarily.
    \item If $|O_{R}^{sampled}|=1$ then \( \sqrt{(x_i - x_0)^2 + (y_i - y_0)^2} \leq \delta \) and \( (x_i, y_i, \theta_i) \sim U \), imposing that the second object must be sampled uniformly at a position that is distanced from the first by at most \( \delta \).
    \item If $|O_{R}^{sampled}|>1$ then \( (x_i, y_i) = \gamma\big(\frac{|O_{R}^{sampled}|}{|O_{R}|}, \kappa(p_0, p_1)\big) + \epsilon \) where \( \epsilon \sim G(0, \sigma) \), here $G$ represents a Gaussian distribution \color{black} with mean zero \color{black} and variance $\sigma$, \color{black} and \color{black}  $\theta = atan2(1, \gamma'(\frac{K}{N}))$. \color{black} \color{black}$\theta$  represented the rotation angle of the object.
\end{itemize}

Our parametric geometrical representation enables us to model any geometric shapes that can be written into the format of a parametric curve. In our current implementation, we defined shapes such as ``line," ``circle," ``rectangle," ``tower," ``spatial:left," ``spatial:right" and so on. Due to space constraints, we refrain from elaborating on the definitions of $\gamma$ and $\kappa$ for all these predefined patterns.  Fig.~\ref{fig:prior} illustrates an example of a parametric geometric prior. Noticeably, we divide patterns into ``ordered" and ``unordered" based on whether the pattern requires an execution sequence.






Note that in our work, language instruction $L$ is typically composed of multiple instructions that deal with different subsets of objects. Specifically, $L$ can be interpreted as a list \color{black} $\{L_i\}$ of sub-instructions $L_i$ \color{black}. For example, $L$ can be a composite instruction: $L = \{L_1, L_2\}$, where $L_1$ refers to placing objects $A$, $B$, and $C$ in a line, and $L_2$ refers to placing object $A$ on the left of $B$. Each sub-instruction $L_i\in L$ is associated with a subset of objects \color{black} $O_{R_i} \subseteq O_R$\color{black}. In our example, $O_{R_1} = \{A,B,C\}$ and $O_{R_2} = \{A,B\}$. In the sequential sampling process described before, we presented the case of a single language instruction for simplicity, but the same process is used for sampling poses given by a complex instruction.

\subsection{Monte-Carlo Tree Search (MCTS) for TAMP}  
\begin{figure}[h]
    \centering
    \includegraphics[width=1.0\linewidth]{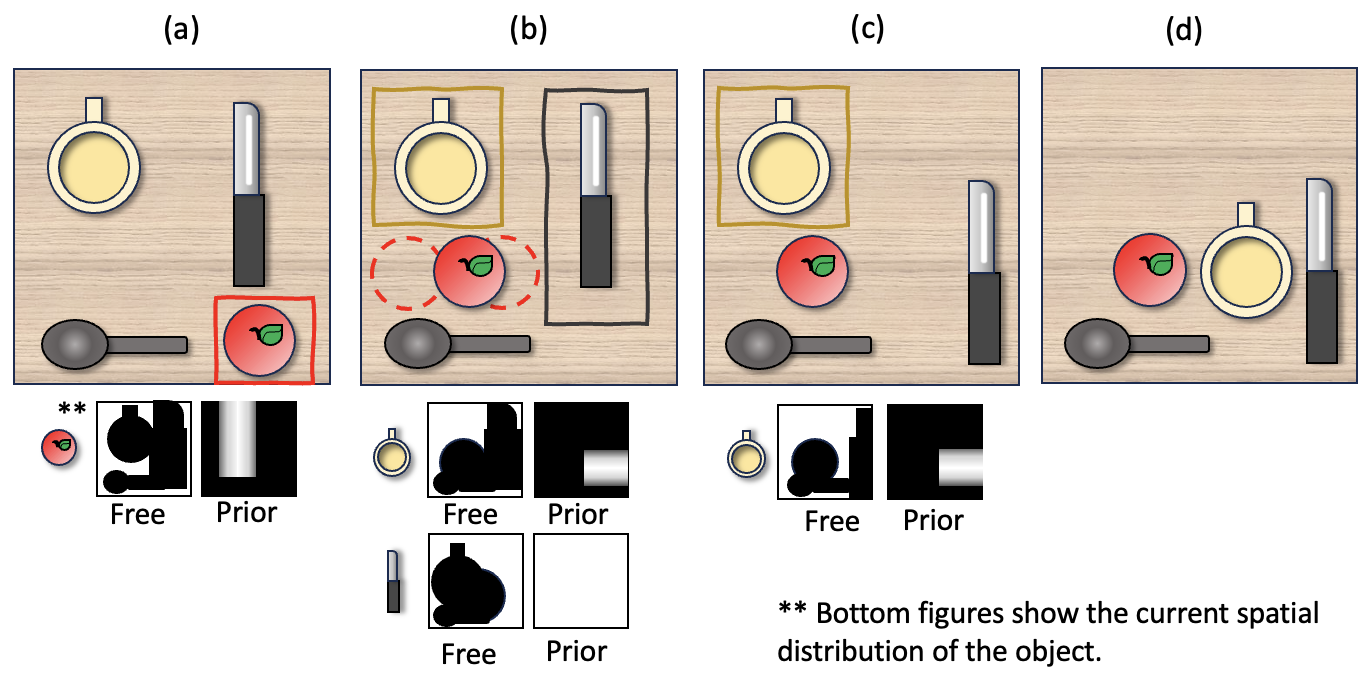}
    \caption{\small{
A minimal example illustrates our MCTS-Planner's aim to arrange a table. The language description provided is: ``Can you please put the apple behind the spoon? And I also want the cup at the right of the apple." The top row displays the current scene arrangement, while the bottom row shows the $f_{prior}$ and $f_{free}$ for the object being manipulated. $f=f_{prior} \times f_{free}$. In spatial distribution figures, black represents probability 0, and white probability 1.
}}
    \label{fig:mcts_exmp}
\end{figure}


\label{sec:mcts}
As previously mentioned, our rearrangement problem can be formulated as a sequential task. In each step, we sample a pose $p_i$ for each object \color{black} $o_i \in O_R$ \color{black} according to a pose distribution $f_{i}$ selected by the LLM and computed as described in Section~\ref{sec:prior}. Once we complete all samplings, all objects will have been placed in their desired locations. However, this task cannot be executed by the robot in a naive sequential order, as the rearrangements made in previous steps may obstruct subsequent sampling. Therefore, we propose a task and motion planner based on the MCTS algorithm to simultaneously arrange the task and address object relocation issues. The objective of our MCTS-Planner is to fulfill all object position requirements defined by $\mathcal{F}$, the spatial distribution list. In the MCTS-Planner, we maintain a tree where each node $s$ in the tree comprises the current objects' poses, $\{p_1, p_2, ..., p_N\}$, and the remaining object spatial requirements $\mathcal{F}_r$, where $\mathcal{F}_r \subseteq \mathcal{F}$.

The MCTS-Planner operates through four phases: selection, expansion, simulation, and back-propagation, as described in Section~\ref{sec:mcts_frame}, adhering to the methodology from \cite{labbe2020monte} except for the simulation stage.
The reward for transitioning to a new state $s$ is quantified by the reduction in the number of spatial requirements, that is, $|\mathcal{F}|-|\mathcal{F}_r|$. This reward structure mirrors the approach in \cite{labbe2020monte}, aiming to steer the search towards branches that efficiently move objects to their goal poses. The simulation phase is elaborated in Algorithm~\ref{alg:Simulation}. This phase begins with the MCTS state $s$ and an attempted action of placing an object $o_i$ in a pose $p_i\sim f_i$, where $f_i$ is the pose distribution explained in the previous section. If the targeted object $o_i$ is not reachable, e.g. there exist some obstacles above it, a reachable obstacle is randomly selected from those above it, and a collision-free pose within the workspace is sampled for relocation (Lines 2-5). If $o_i$ is reachable, an attempt is made to sample its pose with distribution $f_i$ (Line 7). \color{black} If the sampled pose is not collision-free, we will randomly choose an obstacle in the collision and relocate it to \color{black} a collision-free pose within the workspace (Lines 9-13).
Note that obstacle relocation may add previously placed objects back to the requirement list $\mathcal{F}_r$.
However, this obstacle relocation strategy increases the robustness of our rearrangement planner, especially when the environment is cluttered.
Although MCTS operates as an anytime search algorithm, our MCTS-Planner implementation returns the first solution it finds. Furthermore, we have proved that the MCTS-Planner is probabilistically complete within our specified framework in Proposition~\ref{prof:mcts}.

Fig.~\ref{fig:mcts_exmp} presents an illustrative example of the MCTS-Planner at work. Owing to space constraints, we will only explore three simulation steps along the branch that yield a feasible solution. In this scenario, the user requests, ``Can you please put the apple behind the spoon? And I also want the cup to be at the right of the apple." In response, the LLM generates the spatial distribution list $\mathcal{F}=\{f_1, f_2\}$, where $f_1$ is for positioning the apple behind the spoon, and $f_2$ is for placing the cup to the right of the apple. Fig.~\ref{fig:mcts_exmp}(a) illustrates the initial arrangement of objects. Given the dependency of $f_2$ on the apple's position, $K$ actions will be sampled for $f_1$, but none for $f_2$ in this initial setup. Fig.~\ref{fig:mcts_exmp}(b) depicts the outcome of an action of sampling a pose from $f_1$, where the apple is relocated according to $f_1$. The dashed-line circles represent the other $K-1$ actions originating from the root node. After sampling $f_1$, we are left with $\mathcal{F}_r=\{f_2\}$ as shown in Fig.~\ref{fig:mcts_exmp}(b), and an attempt is made to position the cup to the right of the apple. However, the goal position is in collision with the knife, so we need to relocate the knife. Fig.~\ref{fig:mcts_exmp}(c) demonstrates the knife's relocation, maintaining $\mathcal{F}_r=\{f_2\}$. 
Ultimately, Fig.~\ref{fig:mcts_exmp}(d) showcases the final planning result.


\begin{algorithm}
\begin{small}
    \SetKwInOut{Input}{Input}
    \SetKwInOut{Output}{Output}
    \SetKwComment{Comment}{\% }{}
    \caption{Simulation}
		\label{alg:Simulation}
    \SetAlgoLined
		\vspace{0.5mm}
    \Input{$s$: an MCTS state,\\ 
    $f_i$: a pose distribution {\footnotesize (place $o_i$ in a pose $p_i\sim f_i$)}.
    }
    \Output{$(o,p)$: a rearrangement action {\footnotesize (place $o$ in pose $p$)}.}
\vspace{0.5mm}
\If{$o_i$ is not reachable}
{
Select a reachable object $o$ that is blocking object $o_i$;\\
$p\gets$uniformSampling($o,s$);\\
\lIf{$p$ exists}{\Return ($o$,$p$)}
\lElse{\Return \textbf{failure}}
}
\Else{
$p\gets$sampling($o_i, s, f_i$);\\
\lIf{\text{collisionFree$(o_i,p,s)$}}{\Return ($o_i$,$p$)}
\Else{
$o \gets$ Randomly choose an obstacle in collision;\\
$p\gets$\text{uniformSampling$(o,s)$};\\
\lIf{$p$ exists}{\Return ($o$,$p$)}
\lElse{\Return \textbf{failure}}
}
}

\end{small}
\end{algorithm}

\begin{proposition}\label{prof:mcts}
MCTS-Planner is probabilistic complete.
\end{proposition}
\begin{proof}
In the context of semantic rearrangement with distribution list $\mathcal{F}$, we consider a feasible sequence $A^*$ that moves objects to achieve a final state $A_f$, where $f_{i}(A_f[o_i]) > 0$ for each object $o_i$ in set $O_R$. We assert that as the number of iterations $K$ increases, the probability $p$ that MCTS finds a sequence meeting the goal approaches 1: $\lim_{K \to \infty} p=1$.

First, we prove that there is an action sequence $A^*_0$ whose actions can all be generated by MCTS-Planner.
Note that in MCTS-Planner, an action satisfies either of the two rules: \textbf{R1}: Move an object $o_i$ to goal: it requires $o_i\in O_R$ and the new pose $p$ satisfies goal pose requirements (i.e., $f_i(p)>0$); \textbf{R2}: Obstacle relocation: the old pose $p$ of the moved object makes other objects not reachable or it lies in others' goal regions, i.e., $\exists f\in\mathcal F_r, \ s.t. f(p)>0$ (Line 2, 10).
We construct $A^*_0$ by reordering and deleting actions in $A^*$ as follows: (1) If an action in $A^*$ satisfies \textbf{R1} or \textbf{R2}, we add the action to $A^*_0$. 
Otherwise, we delay the addition until it satisfies the rules; 
(2) If the action still cannot satisfy the rules before we examine the next action in $A^*$ for the same object, we delete the former action.
This process ensures $A^*_0$ comprises only MCTS-Planner viable actions, leading to the desired arrangement $A_f$.

Next, we prove that the probability for MCTS to find an action sequence in the ``neighborhood'' of $A^*_0$ approaches 1 as $K$ increases. 
For each intermediate state $s$ of $A^*_0$, denote $(A^*_0[s].o, A^*_0[s].p)$ the rearrangement action that $A^*_0$ chooses at $s$.
Let 
$r:=(\min_{1 \leq i\leq |A^*_0|} C_i)/2|A^*_0|$, where $C_i$ is the minimum distance making the placement pose of the $i^{th}$ action of $A^*_0$ invalid (out of goal distribution or in collision).
For an intermediate state $s$ of $A^*_0$, let $p_r^s$ be the probability that after $K$ actions, there is an action moving $A^*_0[s].o$ to the $r-$neighborhood of $A^*_0[s].p$.
We have
$$
    \lim_{K \to \infty} p \geq \lim_{K \to \infty} \prod_{s\in A^*_0}p_r^s \geq \lim_{K \to \infty} (1-(1-\dfrac{\pi r^2}{| W| N})^K)^{|A^*_0|}
$$
where $|W|$ is the size of the workspace. Specifically, the second function indicates the probability of finding a feasible action sequence, where the intermediate state object poses are maintained in the neighborhood of those in $A^*_0$.
In the third function, 
$(\pi r^2)/(|W|N)$ is the lower bound of the probability of choosing an action in $s$, moving $A^*_0[s].o$ to the $r-$neighborhood of $A^*_0[s].p$ and the base of the $|A^*_0|$ exponent is a lower bound of $p_r^s$.
Since $K$ is the only variable in the third function, $\lim_{K \to \infty} p=1$.
\end{proof}

%% file: experiments.tex
\section{EXPERIMENTS}

\begin{figure*}
  \centering
  \includegraphics[width=0.9\linewidth]{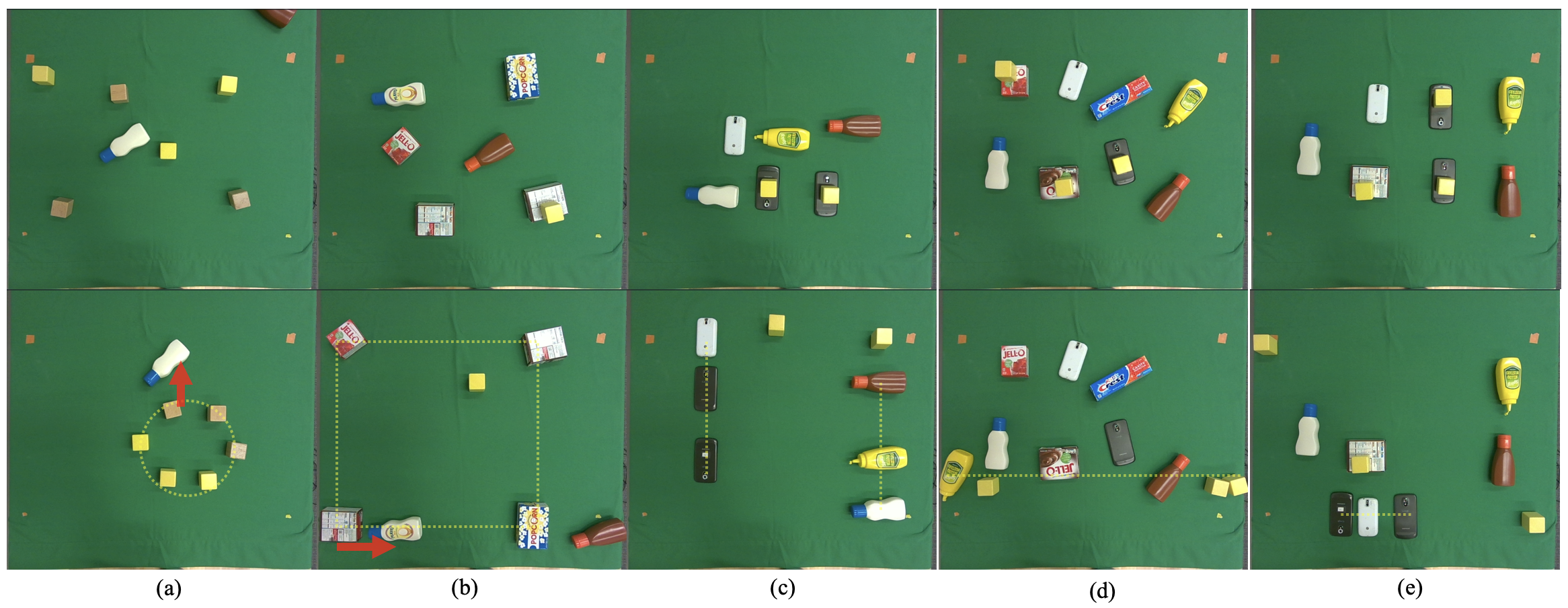}
  \caption{
  \small{Real world demonstration with a UR5e robot. The language instructions for the five scenes are: (a) ``Move all blocks into a circle; while put the white bottle behind one block;"
  (b) ``Put all boxes into a rectangle; and move the white bottle to the right of one box;" (c) ``Move bottles into a line; and formulate all phones into another line;" (d) ``Formulate all yellow objects into a line;" (e) ``Set all phones into a line;". The top row images show the initial scenes and the bottom ones show the results of using LGMCTS on the UR5e.
  Dotted lines imply a shape pattern and red arrows indicate a spatial pattern (left, right, front, back). These real robot experiments show that LGMCTS can parse complex language instructions and also deal with infeasible start configurations as well as pattern composition.}
   }
  \label{fig:robot_evaluation}
  \vspace{-5mm}
\end{figure*}


\subsection{Baselines}

We compare our approach with the following baselines.

\noindent\textbf{Structformer~\cite{liu2022structformer}}. It is a multi-modal transformer specifically designed for language-guided rearrangement tasks.

\noindent\textbf{StructDiffusion~\cite{struct_diffusion}}. It employs a diffusion model combined with a learning-based collision checker for pattern pose generation.

\noindent\textbf{LLMs as Few-Shot Planners~\cite{code-as-policies, singh2023progprompt}}. 
We integrate {\it Code as Policies} and {\it Progprompt} into our evaluation pipeline, where the former generates policy code and the latter Pythonic code. As we cannot directly use the generated code, to streamline the input to our TAMP planner, our setup modifies the output as a sequence of actions (object IDs and their target poses). Initially, LLM processes complete scene details—including object names, IDs, textures, initial poses, and region boundaries for the rearrangement. We then instruct the LLM to take the natural language command and produce the optimal action sequence that contains an ordered list of object IDs and goal poses of the considered objects for rearrangement.
However, for evaluating the Structformer dataset, we consider the structured goal specification pre-available in the dataset as the input to the LLM as it can infer
the action plan from this intermediate representation. This approach avoids redundancy, as generating a natural language command would just restate the same specification for the goal rearrangement. In evaluating both datasets,  we provide a few scenes as examples where the output format is clearly defined to the LLM with ground-truth optimal sequence. This baseline is named LLMs as Few-Shot Planners (LFSP). 

\noindent\textbf{Pose+MCTS}. The Pose+MCTS (PMCTS) approach assumes that a collision-free and semantically aligned goal pose is provided. However, direct execution of this pose might be hindered if the target space is already occupied. To address this, we utilize MCTS to search for a viable plan to place objects in their predetermined goal poses. MCTS is only used as a motion planner. This method follows a two-step approach of using goal poses independently of task planning.

\begin{table}[]
\centering
\scalebox{0.85}{
\begin{tabular}{@{}cccccc@{}}
\toprule
    Method       & Line (4295) & Circle (3416) & Tower (1335)        & Dinner (2440) \\ \midrule
LFSP*~\cite{singh2023progprompt, code-as-policies}             &  41.16\%    &    51.75\%    &   88.80\%   &    27.05\%     \\
Structformer~\cite{liu2022structformer}     &  47.24\%    &    62.64\%    &   99.10\%           &    28.36\%    \\
StructDiffusion~\cite{struct_diffusion}  &  61.49\%    &    81.41\%   &    98.95\%                  &      69.38\%      \\
\textbf{LGMCTS (Ours)}    &    \textbf{95.99\%}         &    \textbf{95.25\%}           &       \textbf{100\%}  &  \textbf{100\%}             \\ \bottomrule
\end{tabular}
}
\caption{\small{Efficacy of LFSP, Structformer, StructDiffusion, and LGMCTS across diverse rearrangement tasks (task counts indicated) from the Structformer dataset. *Due to budget constraints, the LLM baseline LFSP are evaluated on 1150 (10\%) randomly selected scenes of the Structformer dataset.}}
\label{tab:performance}
\end{table}

\subsection{Structformer Dataset}

\begin{figure}
    \centering
    \includegraphics[width=0.75\linewidth]{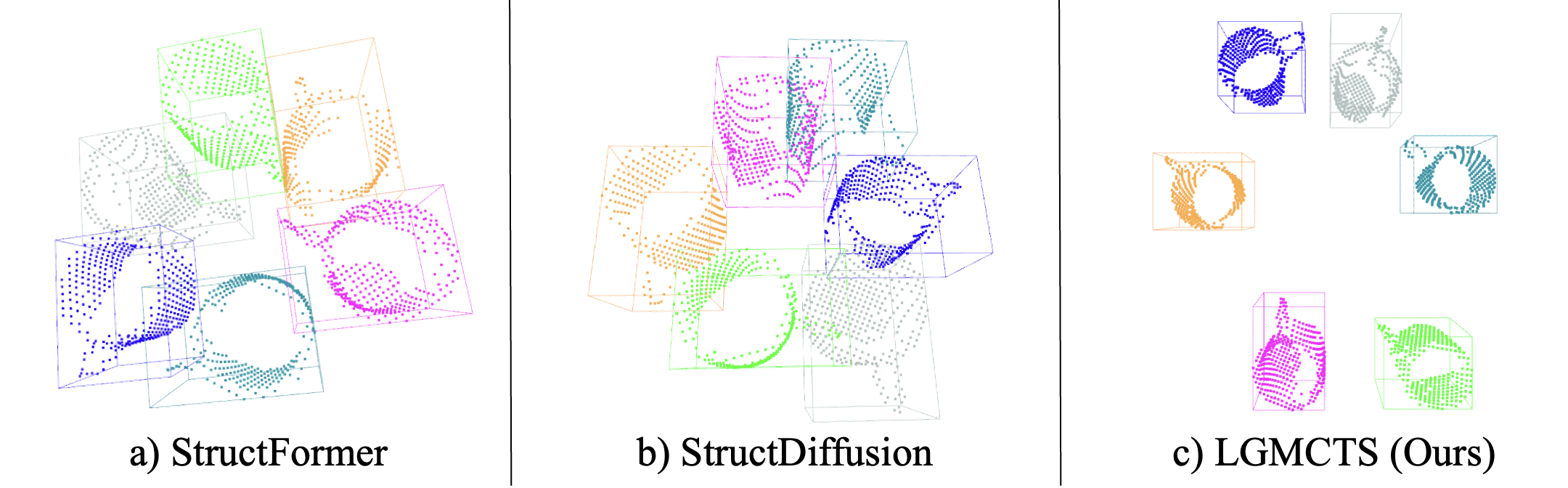}
    \caption{\small{Compared to Structformer and StructDiffusion, LGMCTS ensures a collision-free goal arrangement in all experiments.}}
    \label{fig:struct-result}
\end{figure}
We use the test set from the Structformer dataset to evaluate the goal pose generation ability. This dataset is composed of approximately $11,500$ rearrangement tasks, categorized into four patterns: line, circle, tower, and dinner.  A rearrangement plan is considered successful if it adheres to language constraints and is collision-free, except in the ``tower" task where collisions are inevitable. The ``dinner" task is approached as a composition of patterns, involving the arrangement of items like plates, bowls, and utensils into a ``tower" for plates and bowls, with other items lined up beside it. In both Structformer and StructDiffusion's experimental setup, object selection for rearrangement is based on the object's shape and size. Our evaluation setup does not involve object selection based on shape and size. Hence, to adapt them to our evaluation setup, we provide those two baselines with ground-truth object selection. 
Since the tasks already specify which objects to rearrange for the single pattern rearrangement, based on language instructions, we did not use the LLM parser in LGMCTS for this dataset.

As shown in TABLE~\ref{tab:performance}, LGMCTS demonstrated superior performance in all four rearrangement task categories, achieving remarkable success rates as follows: 95.99\% for ``line", 95.25\% for ``circle", and 100\% for both ``tower" and ``dinner". LSFP performs the least among the baselines due to the inability of LLMs to produce goal patterns with high geometric fidelity.
While StructDiffusion showed improvement over Structformer, it did not match LGMCTS's effectiveness.
For a visual comparison, Fig.~\ref{fig:struct-result} illustrates LGMCTS's success in a ``circle" task scene, highlighting its more actionable goal poses that contribute to higher rearrangement success rates. Conversely, Structformer and StructDiffusion tend to generate goal arrangements with a higher collision risk, leading to lower success rates. 
The evaluation of this dataset strictly assesses the accuracy of collision-free geometric patterns in goal poses, disregarding the executability of plans—a notable limitation of the dataset. 
This limitation is addressed through our proposed ELGR-Benchmark (refer to Section~\ref{sec:elgr}). 
PMCTS method is introduced through our benchmark to solely verify the executability of plans.
As PMCTS deals with motion planning using collision-free ground-truth poses, it was not included in further comparisons on this dataset due to its inherent access to goal poses.
\color{black}

\subsection{ELGR-Benchmark}
\label{sec:elgr}
Existing datasets for semantic object rearrangement, such as Structformer, are limited in that they typically feature only one pattern per scene and do not include crowded scenarios. They also fail to address the challenge of feasibility, particularly when starting configurations are infeasible like one object being placed under another. To bridge these gaps, we introduce ELGR-Bench (\textbf{E}xecutable \textbf{L}anguage-\textbf{G}uided \textbf{R}earrangement \textbf{Bench}mark), which incorporates scenarios with infeasible starting configurations, including tasks that require unstacking and appropriate placement of unstacked objects before the actual rearrangement.
Importantly, this new benchmark presents a novel task termed the ``multi-pattern task", which requires multiple pattern goals to be satisfied during the rearrangement process. In this benchmark, we are considering common shapes such as ``line", ``circle", ``rectangle" and ``spatial" (left/right, front/behind, left/right + front/behind). For each scene, we randomly compose two of the aforementioned patterns and create the multi-pattern task. Success is measured based on the executability of the generated plan and its adherence to semantic requirements. ELGR-Bench builds upon the VIMA-Benchmark~\cite{vima}. 


\begin{table}[]
\centering
\begin{tabular}{@{}cccccc@{}}
\toprule
   Method        & $SR_{p}$ &  $SR_{ep}$  &\\ \midrule
LFSP~\cite{singh2023progprompt, code-as-policies}  &   \textbf{100\%}   &         45.2\%              \\ 
Structformer~\cite{liu2022structformer}  &   n.a.   &         n.a.              \\ 
StructDiffusion~\cite{struct_diffusion}  &   n.a.   &      n.a.              \\

PMCTS &   82.9\%   &        74.1\%              \\
\textbf{LGMCTS (Ours)}   &  90.9\%   &     \textbf{79.2\%}    \\
\bottomrule
\end{tabular}
\caption{\small{$SR_{ep}$, the executable plan success rate, reflects both planning success and the success of executing these plans, indicating if the final positions of objects meet the criteria set by language-based constraints. $SR_{p}$, a part of $SR_{ep}$, only tracks planning success, with PMCTS and LGMCTS capped at 10,000 planning steps. 
If planning with MCTS exceeds the limit, often due to dense object placement in the scene, the motion planning is considered a failure.
}}
\label{tab-ablation-seq}
\vspace{-5mm}
\end{table}

In our benchmark, we compared LGMCTS against two baselines: LFSP and PMCTS, excluding Structformer and StructDiffusion as they cannot handle composite geometric patterns. LFSP, leveraging an LLM, plans goal poses and action sequences simultaneously, while PMCTS, follows a two-step method using a given goal pose and then using MCTS for action planning. LGMCTS uniquely combines goal generation with action planning, aiming for more executable outcomes. As shown in TABLE~\ref{tab-ablation-seq}, LFSP demonstrates a 100\% planning success rate through LLM's capability to generate action plans based on natural language commands and scene context. Nonetheless, over 50\% of these plans are inexecutable, as indicated by the $SR_{ep}$ scores. LGMCTS, however, manages a 90\% success rate in generating action plans, with about 80\% being executable. This performance not only underlines the limitations of LLMs in direct TAMP solving but also showcases LGMCTS's advantage over the two-step PMCTS approach, even when PMCTS is provided with accurate and feasible goal poses.

\subsection{Real Robot Experiments}
We qualitatively evaluated our system using a UR5e robot equipped with a D455 depth camera. The setup of the robot is shown in Fig.~\ref{fig:robot}. We employed the Recognize-Anything-Model (RAM)\cite{zhang2023recognize, huang2023tag2text} and an HSV-based color detector to detect object semantics and colors. Selected queries and their corresponding execution outcomes are presented in Fig.~\ref{fig:robot_evaluation}. We considered five different language instructions involving various objects and initial configurations. For example, Fig.~\ref{fig:robot_evaluation}(b) illustrates the experiment with ``Put all boxes into a rectangle, and move the white bottle to the right of one box." This experiment involves pattern composition, requiring simultaneous consideration of ``line" and ``to the right of" constraint. Additionally, this scene presented an infeasible initial configuration, necessitating the removal of the yellow block before moving the gelatin box. Each experiment presented distinct challenges; for more details, refer to Fig.~\ref{fig:robot_evaluation}. These real-world robot experiments underscore the capabilities of LGMCTS in complex real-world settings.

%% file: conclusion.tex
\section{CONCLUSION}

We introduced LGMCTS, a new framework for tabletop, semantic object rearrangement tasks. LGMCTS stands out by accepting free-form natural language input, accommodating multiple pattern requirements, and jointly solving goal pose generation and action planning. Its main limitation is the extended execution time for complex scenes, highlighting the need for improved tree search efficiency. Future research should focus on adapting LGMCTS to more complex rearrangement scenarios.